\documentclass{article}
\usepackage[utf8]{inputenc}
\usepackage[T1]{fontenc}
\usepackage[margin=1in]{geometry}

\usepackage{biblatex}
\bibliography{main} 

\usepackage{times}
\usepackage{amsmath, amsthm, amssymb, amsfonts}
\title{Can a Transformer Represent a Kalman Filter?}

\author{Gautam Goel \and Peter Bartlett}
\date{Simons Institute, UC Berkeley }

\newtheorem{theorem}{Theorem}

\newtheorem{fact}{Fact}

\DeclareMathOperator*{\expect}{\mathbb{E}}

\begin{document}

\maketitle

\begin{abstract}
 Transformers are a class of autoregressive deep learning architectures which have recently achieved state-of-the-art performance in various vision, language, and robotics tasks. We revisit the problem of Kalman Filtering in linear dynamical systems and show that Transformers can approximate the Kalman Filter in a strong sense. Specifically, for any observable LTI system we construct an explicit causally-masked Transformer which implements the Kalman Filter, up to a small additive error which is bounded uniformly in time; we call our construction the Transformer Filter. Our construction is based on a two-step reduction. We first show that a softmax self-attention block can exactly represent a Nadaraya–Watson kernel smoothing estimator with a Gaussian kernel. We then show that this estimator closely approximates the Kalman Filter. We also investigate how the Transformer Filter can be used for measurement-feedback control and prove that the resulting nonlinear controllers closely approximate the performance of standard optimal control policies such as the LQG controller. 
\end{abstract}

\section{Introduction}
Transformers are a class of autoregressive deep learning architectures designed for various sequence modelling tasks, first introduced in \cite{vaswani2017attention}. Transformers have quickly emerged as the best performing class of deep learning models across a variety of challenging domains, including computer vision, natural language processing, and robotics, and have also been studied in the context of reinforcement learning and decision-making (e.g. \cite{chen2021decision, lee2023supervised, lin2023transformers, zheng2022online}). While the empirical successes of Transformers are exciting, we still lack a formal theory that explains what Transformers can do and why they work.
In this paper, we study how Transformers can be used for filtering and control in linear dynamical systems. We ask perhaps the most basic question one could ask: can a Transformer be used for Kalman Filtering? The Kalman Filter is foundational in optimal control and a crucial component of the Linear-Quadratic-Gaussian (LQG) controller. If Transformers were unable to perform Kalman Filtering, then the use of Transformers in signal processing and control would be suspect; conversely, establishing that Transformers can indeed perform Kalman Filtering is a crucial first step towards establishing the viability of Transformers in these domains. 

In the mathematical theory of deep learning, three questions naturally arise. First, which functions can a given deep learning architecture represent? Second, when trained on data, what function does the deep learning system actually learn? Lastly, how well does this learned function generalize on new data? We focus on the first of these questions and leave the other two for future work. Specifically, we investigate the following questions.  First, is the nonlinear structure of a Transformer compatible with a Kalman Filter at all? This is not obvious; it is possible \textit{a priori} that no matter how a Transformer is implemented, the softmax nonlinearity in the self-attention block will cause the state estimates of the Transformer and the Kalman Filter to diverge over time. Second, if it is possible to represent the Kalman Filter with a Transformer, what would that Transformer look like? How should the states and observations be represented within the Transformer? It is known that \textit{positional encoding} improves the performance of Transformers in some tasks - is it necessary for Kalman Filtering? How large must the Transformer must be, e.g., how large must the embedding dimension be, and how many self-attention blocks are required?

\subsection{Key contributions}

We construct an explicit Transformer which implements the Kalman Filter, up to a small additive error; we call our construction the Transformer Filter. Our construction is based on a two-step reduction. First, we show that a self-attention block can exactly represent a Nadaraya–Watson kernel smoothing estimator with a Gaussin kernel. We select a specific covariance in our Gaussian kernel with a system-theoretic interpretation: it measures how closely a previous state estimate matches the most recent state estimate, where the measure of ``closeness" is the $\ell_2$ distance between the one-step Kalman Filter updates using each of the state estimates. The kernel takes as inputs nonlinear embeddings of the previous state estimates and observations; these embeddings have quadratic dependence on the size of the underlying state-space model. In particular, if the state-space model has an $n$-dimensional state and $p$-dimensional observations, the kernel we construct takes as input embeddings of dimension $O((n + p)^2)$. The second step in our construction is to show that this kernel smoothing algorithm approximates the Kalman Filter in a strong sense. Specifically, for every $\varepsilon > 0$, we show that by increasing a temperature parameter $\beta$ in our kernel, we can ensure that the sequence of state estimates generated by the Transformer Filter is $\varepsilon$-close to the sequence of state estimates generated by the Kalman Filter.
A noteworthy aspect of our construction is that it does not use any positional embedding; permuting the history of state estimates and observations has no effect on the state estimates generated in subsequent timesteps.

We next investigate how the Transformer Filter can be incorporated into a measurement-feedback control system. A key technical challenge is to understand the closed-loop dynamics that are induced by the Transformer Filter; since the state-estimates produced by the Transformer Filter are a nonlinear function of the observations, the resulting closed-loop map is also nonlinear. This means that standard techniques for establishing stability of the system, such as bounding the eigenvalues of the closed-loop map, cannot be used. We show that the Transformer Filter can closely approximate an LQG controller, in the following sense: for every $\varepsilon > 0$, we construct a controller using the Transformer Filter which generates a state sequence that is $\varepsilon$-close to the state sequence generated by the LQG controller. A consequence of this result is that the controllers we construct are weakly stabilizing in the following sense; while they may not drive the state all the way to zero, they are guaranteed to drive the state into a small ball centered at zero. Our result also implies that the cost incurred by our new controller can be driven arbitrarily close to the optimal cost achieved by the LQG controller. All of our approximation results also hold when the reference algorithm is taken to be an $H_{\infty}$ filter or $H_{\infty}$ controller.

\section{Preliminaries}
\subsection{Filtering and Control}
The first problem we consider is \textit{Filtering in Linear Dynamical Systems}. In this problem, we consider a partially observed linear system $$x_{t+1} = Ax_t + w_t, \hspace{5mm} y_t = Cx_t + v_t,$$ where  $x_t \in \mathbb{R}^n$ is an unknown state and $y_t \in \mathbb{R}^p$ is a noisy linear observation of the state; the variables $w_t$ $v_t$ are exogenous disturbances which perturb the state and observation. The state is initialized at time $t = 0$ to some fixed state $x_0$. The task of filtering is to sequentially estimate the state sequence given the observation sequence. We focus on the strictly causal setting, where the filtering algorithm estimates the state $x_t$ after observing $y_0, \ldots, y_{t-1}$. The best-known algorithm for filtering in linear dynamical systems is undoubtedly the \textit{Kalman Filter}, which is the mean-square-optimal linear filter when the disturbances are stochastic. More precisely, if $\{w_t\}_{t\geq 0}$ and $\{v_t\}_{t \geq 0}$ are assumed to be independent, white noise processes, then the estimate $\widehat{x}^{\star}_{t}$ produced by the Kalman Filter satisfies $$\widehat{x}^{\star}_{t} = \inf_z \expect \left[ \|z - x_{t}\|^2 \right],$$ where the infimum is taken over all linear functions $z(y_0, \ldots y_t)$ of the observations. In the special case where the disturbances are Gaussian, the Kalman Filter estimate is also a maximum likelihood estimate of the state conditioned on the observations. The Kalman Filter has the following recursive form: the prediction of the next state given the observations $y_0, \ldots y_{t-1}$ is
\begin{equation} \label{kalman-filter}
    \widehat{x}^{\star}_t = (A - LC)\widehat{x}^{\star}_{t-1} +  Ly_{t-1},
\end{equation}
where $L$ is a fixed matrix called the \textit{Kalman gain} and we initialize $x^{\star}_0 = x_0$. We also note that the $H_{\infty}$ filter has an identical recursive form to the Kalman Filter, except with a different gain matrix $L$ \cite{hassibi1999indefinite}. 

The second problem we consider is \textit{Measurement-Feedback Control in Linear Dynamical Systems}. In this problem, we again consider a partially observed linear system, but the system is now augmented to include a control input $u_t \in \mathbb{R}^m$:
$$x_{t+1} = Ax_t + Bu_t +  w_t, \hspace{5mm} y_t = Cx_t + v_t.$$
The goal of the controller is to select the control action $u_t$ to regulate the state using only  the observations $y_0, \ldots, y_t$. In the Linear-Quadratic-Gaussian (LQG) model, the disturbances  $\{w_t\}_{t\geq 0}$ and $\{v_t\}_{t \geq 0}$ are once again assumed to be independent, white noise processes, and the control actions are selected to minimize the infinite-horizon cost $$\lim_{T \rightarrow \infty} \frac{1}{T} \expect_{\{w_t, v_t\}_{t\geq 0}} \left[\sum_{t = 0}^T x_t^{\top}Qx_t + u_t^{\top}Ru_t \right]. $$ It is known that in this case the optimal policy is to use the Kalman Filter to produce a state estimate $\widehat{x}^{\star}_t$ and then to pick the control actions as a linear function of the estimate \cite{hassibi1999indefinite}. The Kalman Filter estimate is adjusted to account for the influence of the control input, so the LQG policy is 
\begin{equation} \label{measurement-feedback}
    \widehat{x}^{\star}_t = (A + BK - LC)\widehat{x}^{\star}_{t-1} +  Ly_{t-1}, \hspace{1cm} u_t = K \widehat{x}^{\star}_t,
\end{equation}
where $K$ is called the \textit{state-feedback matrix}. 
Other measurement-feedback controllers of the general form (\ref{measurement-feedback}) include the $H_{\infty}$ measurement-feedback controller, which uses a different choice of $L$ and $K$ \cite{hassibi1999indefinite}. 

We assume that the pair $(A, C)$ is observable, and the pair $(A, B)$ is controllable; we refer to \cite{kailath2000linear} for background on linear systems. We let $\| A \|$ denote the spectral norm of a matrix $A$. We make repeated use of the following facts.

\begin{fact}
Let $A \in \mathbb{R}^{q \times q}$ be any stable matrix, i.e., any matrix with spectral radius strictly less than 1. There exist matrices $M, \theta$ such that $A = M\theta M^{-1}$ and $\|\theta \| < 1$.
\end{fact}

\begin{fact}
Let $(L, K)$ represent any stabilizing linear measurement-feedback controller. The matrices $A - LC$  and  $A + BK$ are both stable.
\end{fact}
  
\subsection{Transformers and Softmax Self-Attention}
A Transformer is a deep learning architecture which alternates between self-attention blocks and Multilayer Perceptron (MLP) blocks. In this paper we focus on Transformers with a single self-attention block, followed by a single MLP block; furthermore, we always assume that the weights of the MLP block are chosen so that the MLP block represents the identity function. The interesting part of our construction hence lies in how we choose the parameters of the self-attention block.

A general softmax self-attention block has the following form. It takes as input a series of tokens $q_0, \ldots, q_N$ and a query token $q$, and outputs 
$$F(q_0, \ldots q_N; q) = \frac{\sum_{i = 0}^N \exp{(q^{\top}Aq_i})Mq_i}{\sum_{j = 0}^N \exp{(q^{\top}Aq_j})},$$
where $A$ and $M$ are parameters of the Transformer;  we refer to \cite{phuong2022formal} for an excellent overview of Transformers. In our paper we consider \textit{causally masked} Transformers, which means that we think of the tokens as being indexed by time and at each timestep $t$ we drop all the tokens which have not yet been observed, only keeping those up until time $t$. In our results, we also drop all but the last $H$ observed tokens, to obtain the self-attention block
$$F(q_{t-H+1}, \ldots q_t; q) = \frac{\sum_{i = t-H+1}^t \exp{(q^{\top}Aq_i}) Mq_i}{\sum_{j = t-H+1}^t \exp{(q^{\top}Aq_j})}.$$
In our construction, the tokens $q_i$ are embeddings of the the $i$-th state-estimate and the $i$-th observation, i.e., $q_i = \phi\left( \widehat{x}_{i},  y_i \right),$ where $\phi$ is a nonlinear embedding map. 
The Transformer Filter generates state estimates recursively; it takes as input the past $H$ state estimates and observations  $( \widehat{x}_{t-H}, y_{t - H}), \ldots,   (\widehat{x}_{t-1}, y_{t - 1} )$, embeds them as tokens using the map $\phi$, feeds these tokens $q_{t-H+1}, \ldots, q_t$ into the self-attention block, and outputs a new state estimate $\widehat{x}_t = F(q_{t-H+1}, \ldots q_t; q)$, where we take $q = q_t$. We note that the Kalman Filter has a similar recursive form; it uses the previous estimate $\widehat{x}_{t-1}$ and the previous observation $y_{t-1}$ to generate the new estimate $\widehat{x}_t$. In fact, in the special case when $H = 1$, the Transformer Filter exactly coincides with the Kalman Filter.

\section{Nadaraya–Watson Kernel Smoothing via  Softmax Self-Attention}

Our first result is that the class of Transformers we study is capable of representing a Nadaraya–Watson estimator with a Gaussian kernel. Intuitively, given data $\{z_i\}_{i = 0}^N$ and a query point $z$, a Gaussian kernel smoothing estimator outputs a linear combination of the data, weighted by how close each datapoint $z_i$ is to $z$, where the measure of ``closeness" is determined by a fixed covariance matrix $\Sigma$. We refer to \cite{murphy2012machine} for more background on kernel smoothing and the Nadaraya–Watson estimator.

\begin{theorem} \label{kernel-as-transformer}
Fix $ \Sigma \in \mathbb{R}^{d \times d}$ and $ W \in \mathbb{R}^{k \times d}$. Suppose we are given $z_0, \ldots, z_N \in \mathbb{R}^{d}$ and $z \in \mathbb{R}^{d}$. 
Define the Nadaraya–Watson estimator
\begin{equation} \label{kernel-smoothing}
    F(z_0, \ldots z_N; z) = \frac{\sum_{i = 0}^N \exp{(-(z - z_i)^{\top}\Sigma(z - z_i)}) Wz_i}{\sum_{j = 0}^N \exp{(-(z - z_j)^{\top}\Sigma(z - z_j)})}.
\end{equation}
The function $F$ can be represented by a softmax self-attention block of size $O(d^2H)$. In particular, there exists a nonlinear embedding map $\phi : \mathbb{R}^d \rightarrow \mathbb{R}^{ \ell}$ and matrices $M \in \mathbb{R}^{k \times \ell}$ and $A \in \mathbb{R}^{ \ell \times \ell }$ such that 
$$F(z_0, \ldots z_N; z) = \frac{\sum_{i = 0}^N \exp{(q^{\top}Aq_i)} M q_i}{\sum_{j = 0}^N \exp{(q^{\top}Aq_j)}},$$
where we define $q_i = \phi(z_i)$ and $q = \phi(z)$ and set $\ell  = {n \choose 2} + n + 1$.
\end{theorem}
\begin{proof}
We first show that the function $f: \mathbb{R}^d \times \mathbb{R}^d \rightarrow \mathbb{R}$ given by $$f(u, v) = \exp{(-(u-v)^{\top}\Sigma(u-v))}$$ can be represented as $$ f(u, v) = \exp{(\phi(u)^{\top}A\phi(v)})$$ for some embedding map $\phi$ and matrix $A$ of appropriate dimensions.
Observe that 
\begin{equation}
        (u - v)^{\top} \Sigma (u - v) = \sum_{i, j = 1}^n \Sigma_{i, j} (u_iu_j - 2u_i v_j + v_i v_j). \label{sigma-polynomial} 
    \end{equation}
Define $$\phi(u) = \begin{bmatrix} 1 & u_1 & \ldots & u_n & u_1 u_1 & u_1 u_2 & \ldots & u_n u_{n-1} & u_n u_n \end{bmatrix}^{\top}.$$
It is easy to see that the set $$S(u, v) = \{ \phi(u)^{\top} A \phi(v) \mid A \in \mathbb{R}^{ \ell \times \ell }, \,  A = A^{\top} \}$$ represents all polynomials in $(u_1, \ldots, u_n, v_1, \ldots, v_n)$ of degree at most four such that within each monomial, the degree of the variables appearing in $u$ (resp. $v$) is at most 2. In particular, there exists a choice of $A$ such that $\phi(u)^{\top} A \phi(v)$ is exactly the polynomial (\ref{sigma-polynomial}). We take the matrix $M$ to be the $k \times \ell$ matrix which contains $W$ as a submatrix from columns 2 to $n+1$ and is zero elsewhere; notice that $M\phi(u) = Wu$ for all $u \in \mathbb{R}^d$.
\end{proof}

\section{Filtering}

 We ask: 
\begin{center}
    \textit{Can a Transformer implement the Kalman Filter?}
\end{center}
Naturally, since a Transformer is a complicated nonlinear function of its inputs, it is too much to expect a Transformer to exactly represent a Kalman Filter. We instead ask the following approximation-theoretic question:
for any $\varepsilon > 0$, does there exist a Transformer which generates state estimates which are $\varepsilon$-close to the state estimates generated by the Kalman Filter, uniformly in time? 

We consider the one-layer Transformer whose MLP block is the identity function and whose self-attention block takes as input embeddings of the past $H$ state estimates and observations $$\begin{bmatrix} \widehat{x}_{t-H} \\ y_{t - H} \end{bmatrix}, \ldots,  \begin{bmatrix} \widehat{x}_{t-1} \\ y_{t - 1} \end{bmatrix}$$ and outputs the estimate $$\widehat{x}_t = \sum_{i = t-H+1}^t \alpha_{i, t} \widetilde{x}_i,$$ where we define
$$\alpha_{i, t} = \frac{\exp{(-\beta \|\widetilde{x}_i - \widetilde{x}_t\|^2)}}{\sum_{j=t-H + 1}^t\exp{(-\beta \|\widetilde{x}_j - \widetilde{x}_t\|^2)}}, \hspace{1cm} \widetilde{x}_i = \begin{bmatrix} A - LC & L \end{bmatrix} \begin{bmatrix} \widehat{x}_{i-1} \\ y_{i - 1} \end{bmatrix} ,$$ 
for all $t \geq 1$ and set $\widehat{x}_0, \widetilde{x}_0 = x_0$. We adopt the convention that $\widehat{x}_i, y_i = 0$ for all $i < 0$. We call this filter the Transformer Filter; it is easy to check that this filter is a special case of the Gaussian kernel smoothing estimator (\ref{kernel-smoothing}) and hence by Theorem \ref{kernel-as-transformer} can be represented by a Transformer. The variables $\widetilde{x}_i$ have the following interpretation; they are the estimates that would be generated by the Kalman Filter recursion (\ref{kalman-filter}) if the previous Kalman Filter estimate $\widehat{x}^{\star}_{i-1}$ were replaced by the Transformer estimate $\widehat{x}_{i-1}$. In that sense, the variables $\widetilde{x}_i$ interpolate between the Kalman Filter and the Transformer Filter. We prove:

\begin{theorem} \label{filtering-approx-thm}
For each  $\varepsilon > 0$, there exists a $\beta > 0$ such that the state estimates $\{\widehat{x}_t \}_{t \geq 0}$ generated by the Transformer Filter satisfy $$\|\widehat{x}_t - \widehat{x}^{\star}_{t}\| \leq \varepsilon$$ at all times $t \geq 0$, where $\{\widehat{x}^{\star}_t\}_{t \geq 0}$ are the state-estimates generated by the Kalman Filter (\ref{kalman-filter}). In particular, it suffices to take $$\beta \geq \frac{H^2 \kappa^2}{2e(1-\| \theta \|)^2 \varepsilon^2},$$
where $M, \theta$ are $n \times n$ matrices such that $A - LC = M\theta M^{-1}$ and $\|\theta \| < 1$, and we define $\kappa = \|M\| \|M^{-1}\| $.
\end{theorem}

\begin{proof}
We first show that for all $\varepsilon_1 > 0$, there exists a $\beta$ such that $$\|\widehat{x}_t - \widetilde{x}_{t}\| \leq \varepsilon_1$$ at all times $t \geq 0$. Fix any $\varepsilon_1 > 0$ and any $t \geq 1$. 
Notice that for each $i \in \{t-H+1, \ldots, t\}$, the following inequality holds:
$$\alpha_{i, t} < \exp{(-\beta \|\widetilde{x}_i - \widetilde{x}_t\|^2)}.$$
This is because $$\sum_{j = t-H+1}^t \exp{(-\beta \|\widetilde{x}_j - \widetilde{x}_t\|^2)} > 1. $$
It follows that
\begin{eqnarray*}
\| \widehat{x}_t - \widetilde{x}_t\| &=& \left\| \sum_{i = t-H+1}^t \alpha_{i, t} (\widetilde{x}_i - \widetilde{x}_t) \right\| \\
&\leq&  \sum_{i = t-H+1}^t \alpha_{i, t} \| \widetilde{x}_i - \widetilde{x}_t \| \\
&<&  \sum_{i = t-H+1}^t \exp{(-\beta \|\widetilde{x}_i - \widetilde{x}_t\|^2)} \|\widetilde{x}_i - \widetilde{x}_t\| \\
&\leq& H \max_{\gamma \geq 0} \exp{(-\beta \gamma^2)} \gamma, 
\end{eqnarray*}
where we used the fact that $\sum_{i = t-H+1}^t \alpha_{i, t} = 1$ in the first step. It is easy to check that the function $f(\gamma) = He^{-\beta \gamma^2}\gamma$ is strictly increasing in the interval $(0, (2\beta)^{-1/2})$ and strictly decreasing in the interval $((2\beta)^{-1/2}, \infty)$ and hence takes its maximum value of $He^{-1/2}(2\beta)^{-1/2}$ at $\gamma = (2\beta)^{-1/2}$. It follows that $\|\widehat{x}_t - \widetilde{x}_{t}\| \leq \varepsilon_1$ as long as $\beta \geq \frac{H^2}{2e \varepsilon_1^2}.$

We now show that this result implies Theorem \ref{filtering-approx-thm}. Fix $\varepsilon > 0$ and any $t \geq 0$. Set  $\varepsilon_1 = \frac{(1-\gamma)\varepsilon}{\kappa}$ and $\beta \geq \frac{H^2}{2e \varepsilon_1^2}.$ Using the preceding argument, this suffices to ensure that $\|\widehat{x}_t - \widetilde{x}_t \| \leq \varepsilon_1$ for all $t \geq 0$. We see that 
\begin{eqnarray*}
    \|\widehat{x}_t - \widehat{x}^{\star}_t \| &\leq& \|\widehat{x}_t - \widetilde{x}_t \| + \|\widetilde{x}_t - \widehat{x}^{\star}_t \| \\
    &\leq& \|\widehat{x}_t - \widetilde{x}_t \| +  \|(A - LC)(\widehat{x}_{t-1} - \widehat{x}^{\star}_{t-1}) \|
\end{eqnarray*}
Proceeding recursively, we obtain the bound 
\begin{eqnarray*}
\|\widehat{x}_t - \widehat{x}^{\star}_t \| &\leq& \sum_{i = 0}^t \|(A - LC)^i \| \|\widehat{x}_{t-i} - \widetilde{x}_{t-i} \| \\
&\leq& \varepsilon_1 \sum_{i = 0}^t \|(A - LC)^i \|.
\end{eqnarray*}
Using the fact that $A - LC = M\theta M^{-1}$ with $\|M\|\|M^{-1}\| = \kappa$ and $\|\theta \| < 1$, we see that 
\begin{eqnarray*}
\|\widehat{x}_t - \widehat{x}^{\star}_t \| &\leq& \varepsilon_1 \sum_{i = 0}^t \|(M\theta M^{-1})^i \| \\
&\leq& \kappa \varepsilon_1 \sum_{i = 0}^t \| \theta \| ^i \\
&\leq& \frac{\kappa \varepsilon_1}{1-\| \theta \|} \\
&=& \varepsilon.
\end{eqnarray*}
\end{proof}
We note that the only property of the gain matrix $L$ we used is that $A - LC$ is stable; since this property also holds for the $H_{\infty}$-optimal choice of $L$, our proof also shows that a Transformer can approximate an $H_{\infty}$-optimal filter.

\section{Control}
We ask: 
\begin{center}
    \textit{Can the Transformer Filter be used in place of the Kalman Filter in the LQG controller?}
\end{center}
Since the Transformer Filter only represents the Kalman Filter approximately, we cannot hope to implement the LQG controller exactly. Instead, we ask if the closed-loop dynamics generated by the Transformer can closely approximate the closed-loop dynamics generated by the LQG controller in the following sense: for any $\varepsilon > 0$, can we guarantee that the states generated by the Transformer are $\varepsilon$-close to the states generated by the Transformer, uniformly in time? We emphasize that this is far from obvious, and in particular does not follow directly from Theorem \ref{filtering-approx-thm}. Even if the state-estimates generated by the Transformer Filter are close to those generated by the Kalman Filter, it does not automatically follow that the resulting control policies will generate similar state trajectories. This is because any difference in state estimates will lead to a difference in the control actions, which in turn affects future states, future observations, and so on; in effect, minute deviations between the two state estimates could be amplified over time, leading to diverging trajectories. In order to show that this scenario does not occur, we need to analyze the stability of the closed-loop map induced by the Transformer Filter. This is challenging, because this map is nonlinear, and hence we cannot use standard techniques from linear systems theory.


We consider the controller given by $$u_t = K\widehat{x}_t $$ where we set $$\widehat{x}_t = \sum_{i = t-H+1}^t \alpha_{i, t} \widetilde{x}_i,$$ and we define
$$\alpha_{i, t} = \frac{\exp{(-\beta \|\widetilde{x}_i - \widetilde{x}_t\|^2)}}{\sum_{j=t-H + 1}^t\exp{(-\beta \|\widetilde{x}_j - \widetilde{x}_t\|^2)}}, \hspace{1cm} \widetilde{x}_i = (A + BK - LC)\widehat{x}_{i -1} + Ly_{i-1}.$$
We initialize the state of the system driven by the Transformer system to match the state of the system driven by the LQG policy (i.e., $x_0 = x_0^{\star}$) and similarly initialize the state estimates to be the same ($\widehat{x}_0 = \widehat{x}_0^{\star}$). We also initialize $\widetilde{x}_0 = \widehat{x}_0$. We prove:

\begin{theorem} \label{control-approx-thm}
For each  $\varepsilon > 0$, there exists a $\beta > 0$ such that the states $\{x_t\}_{t \geq 0}$ generated by the Transformer Filter satisfy $$\|x_t - x_t^{\star}\| \leq \varepsilon$$ at all times $t \geq 0$, where $\{x_t^{\star}\}_{t \geq 0}$ are the states generated by the optimal LQG control policy (\ref{measurement-feedback}). In particular, it suffices to take $$\beta \geq \frac{C H^2 \kappa^2}{2e (1 - \|\Theta\|)^2 \varepsilon^2} $$ 
where we define 
$$C = 2\|BK\|^2 + 2\|A + BK - LC \|^2, \hspace{1cm} \kappa = \|\mathbb{M}\| \| \mathbb{M}^{-1}\|,$$ 
and  $\mathbb{M}, \Theta$ are $4n \times 4n$ matrices such that $\mathbb{A} = \mathbb{M}\Theta \mathbb{M}^{-1}$  and $\|\Theta \|  < 1$, and we define 
$$\mathbb{A} = \begin{bmatrix}  A & BK & 0 & 0 \\ LC & A + BK - LC  &  0 & 0  \\
0 & 0  &  A & BK \\  0 & 0 & LC & A + BK - LC  \end{bmatrix}.$$
\end{theorem}

Before we turn to the proof, we note an interesting consequence of this result: the controller induced by the Transformer Filter is \textit{weakly stabilizing} in the sense that no matter how $x_0$ is chosen, if the disturbances are zero then the states generated by the controller will eventually be confined to a ball of radius $\varepsilon$ centered at the origin. This follows from the fact that the LQG controller is stabilizing (i.e., it drives the state to zero in the absence of noise).

\begin{proof} An identical argument to that appearing in the proof of Theorem \ref{filtering-approx-thm} establishes that for all $\varepsilon_1 > 0$, choosing $$\beta \geq \frac{H^2}{2e \varepsilon_1^2}$$ guarantees that $$\|\widehat{x}_t - \widetilde{x}_{t}\| \leq \varepsilon_1$$ at all times $t \geq 0$. The closed-loop dynamics can be written as 
\begin{equation} \label{closed-loop-dynamics}
\begin{bmatrix} x_{t+1} \\ \widetilde{x}_{t+1} \\ x_{t+1}^{\star} \\ \widehat{x}_{t+1}^{\star} \end{bmatrix}  = \mathbb{A} \begin{bmatrix} x_t \\ \widetilde{x}_t \\ x_t^{\star} \\ \widehat{x}_t^{\star} \end{bmatrix} + \eta_t + \nu_t
\end{equation}
where we set $$\eta_t = \begin{bmatrix} BK (\widehat{x}_t - \widetilde{x}_t) \\ (A + BK - LC) (\widehat{x}_t - \widetilde{x}_t) \\ 0 \\ 0 \end{bmatrix}, \hspace{1cm} \nu_t = \begin{bmatrix} w_t \\ Lv_t \\ w_t \\ Lv_t \end{bmatrix}$$ for all $t \geq 0$. We emphasize that while the dynamics (\ref{closed-loop-dynamics}) may superficially appear linear, the variable $\eta_t$ depends on $\widehat{x}_t$, which itself is a nonlinear function of the $H$ observations $y_{t-H+1}, \ldots, y_t$, so that the overall behavior of the closed-loop system is nonlinear. In effect, we have pushed all of the nonlinearity and memory in the closed-loop system into the variables $\{\eta_t\}_{t \geq 0}$. 

The matrix $\mathbb{A}$ is stable. To see this, notice that $$\mathbb{A}  = \mathbb{Q}^{-1}\mathbb{S} \mathbb{Q},$$ where we define the $4n \times 4n$ block matrices $$\mathbb{Q} = \begin{bmatrix} I & -I & 0 & 0 \\ 0 & I & 0 & 0 \\ 0 & 0 & I & -I \\ 0 & 0 & 0 & I \end{bmatrix}, \hspace{1cm} \mathbb{S} = \begin{bmatrix} A - LC & 0 & 0 & 0 \\ LC & A + BK & 0 & 0 \\ 0 & 0 & A - LC & 0 \\ 0 & 0 & LC & A + BK \end{bmatrix}.$$
It is clear that $\mathbb{S}$ is stable, because $\mathbb{S}$ is block lower-triangular and the matrices  $A - LC$ and $A + BK$ appearing on the diagonal of $\mathbb{S}$ are both stable. Since $\mathbb{A}$ is similar to $\mathbb{S}$ and $\mathbb{S}$ is stable, $\mathbb{A}$ must also be stable. It follows that $\mathbb{A} = \mathbb{M} \Theta \mathbb{M}^{-1}$ for some $4n \times 4n$ matrices $\mathbb{M}$ and $\Theta$ such that $\|\Theta\| < 1$. 

Fix $t \geq 1$ and $\varepsilon > 0$. Set $$\varepsilon_1 = \frac{\varepsilon (1 - \|\Theta\|)}{\kappa  \sqrt{2\|BK\|^2 + 2\|A + BK - LC \|^2}}, \hspace{1cm} \beta \geq \frac{H^2}{2e \varepsilon_1^2}. $$ 
The closed-loop dynamics (\ref{closed-loop-dynamics}) imply that 
$$\begin{bmatrix} x_t \\ \widetilde{x}_t \\ x_t^{\star} \\ \widehat{x}_t^{\star} \end{bmatrix} = \mathbb{A}^t \begin{bmatrix} x_0 \\ \widetilde{x}_0 \\ x_0^{\star} \\ \widehat{x}_0^{\star} \end{bmatrix} + \sum_{i = 0}^{t-1} \mathbb{A}^{t-1-i}(\eta_i + \nu_i). $$
It follows that 
$$
 x_t - x_t^{\star}  = \begin{bmatrix} I & 0 & -I & 0 \end{bmatrix} \left(\mathbb{A}^t \begin{bmatrix} x_0 \\ \widetilde{x}_0 \\ x_0^{\star} \\ \widehat{x}_0^{\star}\end{bmatrix} + \sum_{i = 0}^{t-1} \mathbb{A}^{t-1-i} \nu_i \right) + \begin{bmatrix}I & 0 & -I & 0 \end{bmatrix} \sum_{i = 0}^{t-1} \mathbb{A}^{t-1-i}\eta_i.
$$
Notice that the first term is zero; this follows from the assumption that we initialize $x_0 = x_0^{\star}$ and $\widetilde{x}_0 = \widehat{x}_0 = \widehat{x}_0^{\star}$ and the block-diagonal structure of $\mathbb{A}$. We see that 
\begin{eqnarray*}
\|x_t - x_t^{\star}\| &\leq& \left\|\begin{bmatrix}I & 0 & -I & 0 \end{bmatrix}\right  \| \cdot  \sum_{i = 0}^{t-1} \|\mathbb{A}^{t-1-i}\| \|\eta_i\| \\
&=& \left\|\begin{bmatrix}I & 0 & -I & 0 \end{bmatrix}\right  \| \cdot  \sum_{i = 0}^{t-1} \left\|\left(\mathbb{M} \Theta \mathbb{M}^{-1} \right)^{t-1-i}\right\| \|\eta_i\| \\
&\leq&  \varepsilon_1 \cdot \kappa \sqrt{2\|BK\|^2 + 2\|A + BK - LC \|^2} \sum_{i = 0}^{t-1} \|\Theta\|^{t-1-i} \\
&\leq&   \frac{\varepsilon_1 \cdot \kappa \sqrt{2\|BK\|^2 + 2\|A + BK - LC \|^2}}{1 - \|\Theta\|} \\
&\leq& \varepsilon,
\end{eqnarray*}
where in the third step we used the fact that $\|\widetilde{x}_i - \widehat{x}_i\| \leq \varepsilon_1$ for all $i \geq 0$.
\end{proof}
We note that the only property of the gain matrices $L$ and $K$ we used is that $A - LC$ and $A + BK$ are stable; since this property also holds for the $H_{\infty}$-optimal choice of $L$ and $K$, our proof also shows that a Transformer can approximate an $H_{\infty}$-optimal measurement-feedback controller.

\printbibliography

@book{hassibi1999indefinite,
  title={Indefinite-Quadratic estimation and control: a unified approach to H2 and H-Infinity theories},
  author={Hassibi, Babak and Sayed, Ali H and Kailath, Thomas},
  year={1999},
  publisher={SIAM}
}

@book{kailath2000linear,
  title={Linear estimation},
  author={Kailath, Thomas and Sayed, Ali H and Hassibi, Babak},
  year={2000},
  publisher={Prentice Hall}
}

@article{phuong2022formal,
  title={Formal algorithms for transformers},
  author={Phuong, Mary and Hutter, Marcus},
  journal={arXiv preprint arXiv:2207.09238},
  year={2022}
}

@article{vaswani2017attention,
  title={Attention is all you need},
  author={Vaswani, Ashish and Shazeer, Noam and Parmar, Niki and Uszkoreit, Jakob and Jones, Llion and Gomez, Aidan N and Kaiser, {\L}ukasz and Polosukhin, Illia},
  journal={Advances in neural information processing systems},
  volume={30},
  year={2017}
}

@book{murphy2012machine,
  title={Machine learning: a probabilistic perspective},
  author={Murphy, Kevin P},
  year={2012},
  publisher={MIT press}
}

@article{chen2021decision,
  title={Decision transformer: Reinforcement learning via sequence modeling},
  author={Chen, Lili and Lu, Kevin and Rajeswaran, Aravind and Lee, Kimin and Grover, Aditya and Laskin, Misha and Abbeel, Pieter and Srinivas, Aravind and Mordatch, Igor},
  journal={Advances in neural information processing systems},
  volume={34},
  pages={15084--15097},
  year={2021}
}

@inproceedings{zheng2022online,
  title={Online decision transformer},
  author={Zheng, Qinqing and Zhang, Amy and Grover, Aditya},
  booktitle={international conference on machine learning},
  pages={27042--27059},
  year={2022},
  organization={PMLR}
}

@article{lin2023transformers,
  title={Transformers as Decision Makers: Provable In-Context Reinforcement Learning via Supervised Pretraining},
  author={Lin, Licong and Bai, Yu and Mei, Song},
  journal={arXiv preprint arXiv:2310.08566},
  year={2023}
}

@article{lee2023supervised,
  title={Supervised Pretraining Can Learn In-Context Reinforcement Learning},
  author={Lee, Jonathan N and Xie, Annie and Pacchiano, Aldo and Chandak, Yash and Finn, Chelsea and Nachum, Ofir and Brunskill, Emma},
  journal={arXiv preprint arXiv:2306.14892},
  year={2023}
}

\end{document}